\documentclass[accepted]{uai2026}

\usepackage[american]{babel}
\usepackage{natbib}
\usepackage{amsmath,amssymb,amsthm}
\usepackage{mathtools}
\usepackage{booktabs}
\usepackage{graphicx}
\usepackage{algorithm}
\usepackage{algorithmic}
\usepackage{subcaption}
\usepackage[font=footnotesize]{caption}
\usepackage{xcolor}
\usepackage{colortbl}
\usepackage{tikz}
\usetikzlibrary{arrows.meta, positioning, calc, fit, decorations.markings}

\definecolor{lightblue}{RGB}{224,240,255}
\definecolor{lightyellow}{RGB}{255,255,224}
\definecolor{lightgreen}{RGB}{224,255,224}
\definecolor{lightred}{RGB}{255,224,224}
\definecolor{accentblue}{HTML}{003E74}
\definecolor{riskred}{HTML}{E74C3C}
\definecolor{protectgreen}{HTML}{27AE60}

\definecolor{impgreen}{RGB}{34,139,34}
\definecolor{impred}{RGB}{200,50,50}
\newcommand{\impup}[1]{\textcolor{impgreen}{#1}}
\newcommand{\impdn}[1]{\textcolor{impred}{#1}}

\newcommand{\E}{\mathbb{E}}
\newcommand{\R}{\mathbb{R}}
\newcommand{\X}{\mathcal{X}}
\newcommand{\D}{\mathcal{D}}

\newcommand{\Kcal}{\mathcal{K}}
\newcommand{\TV}{\mathrm{TV}}

\newcommand{\Wass}{\mathrm{W}}
\newcommand{\doop}{\mathrm{do}}
\DeclareMathOperator*{\argmin}{arg\,min}

\makeatletter
\titlespacing{\section}{\z@}{*0.5}{*0.0}
\titlespacing{\subsection}{\z@}{*0.5}{*0.0}
\titlespacing{\subsubsection}{\z@}{*0.5}{*0.0}
\titlespacing{\paragraph}{\z@}{*0.0}{0.5em}
\makeatother

\theoremstyle{plain}
\newtheorem{theorem}{Theorem}

\newtheorem{proposition}{Proposition}

\newtheorem{assumption}{Assumption}
\theoremstyle{definition}

\newtheorem{remark}{Remark}
\newtheorem{example}{Example}

\title{CausalWrap: Model-Agnostic Causal Constraint Wrappers for Tabular Synthetic Data}

\author[1]{Amir Asiaee}
\author[1]{Zhuohui J.\ Liang}
\author[2]{Chao Yan}
\affil[1]{%
    Department of Biostatistics, Vanderbilt University Medical Center, Nashville, TN, USA
}
\affil[2]{%
    Department of Biomedical Informatics, Vanderbilt University Medical Center, Nashville, TN, USA
}
\affil[ ]{\texttt{\{amir.asiaeetaheri, zhuohui.liang.1, chao.yan.1\}@vumc.org}}

\begin{document}

\setlength{\parskip}{0.25\baselineskip}
\setlength{\abovedisplayskip}{6pt plus 2pt minus 2pt}
\setlength{\belowdisplayskip}{6pt plus 2pt minus 2pt}
\setlength{\abovedisplayshortskip}{3pt plus 1pt}
\setlength{\belowdisplayshortskip}{3pt plus 1pt}
\setlength{\textfloatsep}{6pt plus 2pt minus 2pt}
\setlength{\floatsep}{5pt plus 2pt minus 2pt}
\setlength{\intextsep}{5pt plus 2pt minus 2pt}
\setlength{\abovecaptionskip}{4pt plus 1pt minus 1pt}
\setlength{\belowcaptionskip}{0pt}
\maketitle

\begin{abstract}
Tabular synthetic data generators are typically trained to match observational distributions, which can yield high conventional utility (e.g., column correlations, predictive accuracy) yet poor preservation of structural relations relevant to causal analysis and out-of-distribution (OOD) reasoning. When the downstream use of synthetic data involves causal reasoning---estimating treatment effects, evaluating policies, or testing mediation pathways---merely matching the observational distribution is insufficient: structural fidelity and treatment-mechanism preservation become essential. We propose \emph{CausalWrap} (CW), a model-agnostic wrapper that injects \emph{partial causal knowledge} (PCK)---trusted edges, forbidden edges, and qualitative/monotonic constraints---into any pretrained base generator (GAN, VAE, or diffusion model), without requiring access to its internals. CW learns a lightweight, differentiable \emph{post-hoc correction map} applied to samples from the base generator, optimized with causal penalty terms under an augmented-Lagrangian schedule. We provide theoretical results connecting penalty-based optimization to constraint satisfaction and relating approximate factorization to joint distributional control. We validate CW on simulated structural causal models (SCMs) with known ground-truth interventions, semi-synthetic causal benchmarks (IHDP and an ACIC-style suite), and a real-world ICU cohort (MIMIC-IV) with expert-elicited partial graphs. CW improves causal fidelity across diverse base generators---e.g., reducing average treatment effect (ATE) error by up to 63\% on ACIC and lifting ATE agreement from 0.00 to 0.38 on the intensive care unit (ICU) cohort---while largely retaining conventional utility.
\end{abstract}

\section{Introduction}
\label{sec:intro}

Synthetic data generation has emerged as a critical tool for enabling data sharing, augmentation, and method benchmarking in high-stakes domains such as healthcare \citep{walonoski2018synthea,choi2017medgan}. The standard approach trains a deep generative model---generative adversarial network (GAN) \citep{xu2019ctgan,park2018tablegan}, variational autoencoder (VAE), diffusion model \citep{kotelnikov2023tabddpm}, or large language model (LLM)---to approximate the observational joint distribution $P(X)$ and samples new records from the learned model. Conventional evaluation metrics focus on distributional resemblance (marginal statistics, pairwise correlations, Wasserstein distance) or downstream predictive utility (train-on-synthetic, test-on-real, or TSTR) \citep{yan2024jmir_tutorial,liang2025generating}.

However, when the downstream task involves causal reasoning---estimating treatment effects, evaluating policies under counterfactual scenarios, or testing mediation pathways---matching the observational distribution may be necessary but insufficient. Causal conclusions rely on the data-generating process preserving conditional independence structure, mechanism modularity, and interventional semantics \citep{pearl2009causality,peters2017elements}. A synthetic dataset that closely matches observational statistics but distorts the treatment-assignment mechanism or the outcome model can produce misleading treatment effect estimates and flawed policy recommendations \citep{amad2025steam}.

Recent work has begun addressing this gap from complementary angles: STEAM \citep{amad2025steam} augments generators with treatment-mechanism--specific losses; TabStruct \citep{jiang2025tabstruct} proposes structural fidelity metrics; DCM \citep{chao2023dcm} builds generators directly from causal graphs; and C-DGMs \citep{stoian2024constrained} enforce hard domain constraints via differentiable layers (see Section~\ref{sec:related} for details).

These approaches are valuable but share a common limitation: they either require a \emph{fully specified} causal graph (e.g., DCM) or address only specific types of constraints (STEAM targets treatment mechanisms; C-DGMs target domain-knowledge constraints but not causal structure). In practice, domain experts often possess \emph{partial causal knowledge} (PCK): a subset of trusted directed edges (``smoking causes lung cancer''), forbidden edges (``treatment cannot cause age''), temporal orderings, qualitative monotonicity constraints (``increasing drug dose cannot decrease blood pressure''), and directional effect signs. Crucially, practitioners may only have access to samples from an existing off-the-shelf generator---without the ability to modify its architecture or training procedure---and want to improve its causal properties post hoc.

\paragraph{Contributions.} This paper introduces CausalWrap, a model-agnostic \emph{post-hoc correction} wrapper that injects PCK into samples from any base tabular generator. Our contributions are:
\begin{enumerate}
    \item \textbf{A flexible causal constraint framework.} We formalize partial causal knowledge as trusted edges, forbidden edges, and qualitative/monotonic constraints, and define differentiable penalty functionals based on kernel independence measures and paired-sample comparisons.

    \item \textbf{Theoretical results.} We prove penalty-method convergence to constraint-satisfying optima (Theorem~\ref{thm:penalty}) and a chain-rule TV bound showing that approximate conditional matching implies joint distributional closeness (Proposition~\ref{prop:chain_tv}).

    \item \textbf{An augmented-Lagrangian training schedule.} We extend the basic penalty method to an augmented-Lagrangian scheme with dual updates, improving convergence without requiring manual $\lambda$-tuning.

    \item \textbf{Comprehensive experiments.} We evaluate CausalWrap on simulated SCMs (Tier~1), semi-synthetic causal benchmarks (Tier~2: IHDP and an Atlantic Causal Inference Conference (ACIC)-style suite), and a real-world ICU cohort from MIMIC-IV (Tier~3) using a small clinically-plausible constraint set. We wrap three base generators (CTGAN, TabDDPM, TVAE) and include an oracle SCM sampler for Tier~1 as an upper bound when the full causal model is known.
\end{enumerate}

\section{Related Work}
\label{sec:related}

\subsection{Tabular Synthetic Data Generators}

GAN-based generators include CTGAN \citep{xu2019ctgan}, TableGAN \citep{park2018tablegan}, and CTAB-GAN+ \citep{zhao2024ctabganplus}; VAE-based approaches include TVAE \citep{xu2019ctgan}. Diffusion-based generators such as TabDDPM \citep{kotelnikov2023tabddpm} apply DDPM \citep{ho2020ddpm} with multinomial and Gaussian diffusion for mixed-type columns, achieving state-of-the-art fidelity. Score-based SDEs \citep{song2021score} and LLM-based methods (GReaT \citep{borisov2023great}) provide further alternatives; see \citet{fang2024survey} for a survey.

\subsection{Medical and EHR Synthetic Data}

Synthetic electronic health record (EHR) generation is motivated by regulatory barriers to data sharing and privacy risk. MedGAN \citep{choi2017medgan} generates multi-label discrete patient records using an autoencoder-GAN architecture. EMR-WGAN \citep{zhang2020emrwgan} improves upon this by incorporating the Wasserstein objective and removing the autoencoder, yielding higher-fidelity EHR simulation. Recent tutorials and reviews emphasize evaluation and reproducibility in health data synthesis \citep{yan2024jmir_tutorial}; \citet{liang2025generating} demonstrate clinically grounded synthetic cohort generation for multi-national longitudinal studies.

\subsection{Causal Models and Causal Generative Modeling}

SCMs provide a mechanism-based factorization with interventional and counterfactual semantics \citep{pearl2009causality,peters2017elements}. CausalGAN \citep{kocaoglu2018causalgan} learns causal implicit generative models structured by a given causal graph, enabling both observational and interventional sampling. DCM \citep{chao2023dcm} couples diffusion models with causal graphs for interventional and counterfactual queries. Relational SCM-based synthesis extends to multi-table data \citep{hoppe2025relational}. DeCaFlow \citep{decaflow2025} uses normalizing flows for deconfounding and causal generation.

\subsection{Treatment-Effect--Oriented Synthetic Data and Evaluation}

STEAM \citep{amad2025steam} proposes desiderata (covariate, treatment-assignment, and outcome preservation) and evaluation metrics tailored to downstream treatment-effect analysis in medicine, augmenting a base generator with mechanism-specific losses. TabStruct \citep{jiang2025tabstruct} emphasizes structural fidelity and introduces ``global utility'' as an SCM-free proxy metric. Fairness-aware synthetic generation uses causal or counterfactual constraints \citep{nagesh2025fairtabgen,kusner2017cf_fairness}. Causal-TGAN \citep{wen2021causaltgan} injects a known causal graph into CTGAN's conditional generation. The ``Causality for Tabular Data Synthesis'' (CaTS) benchmark \citep{liu2024cats} evaluates tabular generators on causal metrics using high-order structural causal models.

\subsection{Constrained Deep Generative Models}

Stoian et al.~\citep{stoian2024constrained} show how to transform DGMs into Constrained DGMs (C-DGMs) by adding a differentiable constraint layer that guarantees sample compliance with domain constraints---achieving near-100\% compliance compared to $>$95\% violation rates in unconstrained models. Physics-informed generative models \citep{yang2018pidgm} embed PDE constraints as soft penalties in the generator loss, providing a template for domain-knowledge regularization. Augmented Lagrangian methods for constrained neural network training \citep{fioretto2024auglag,chatzos2023stochastic_alm} provide convergence-guaranteed alternatives to fixed-penalty approaches.

\subsection{Causal Discovery, Independence Testing, and Constrained Optimization}

Causal discovery methods---constraint-based (PC/FCI) \citep{spirtes2000causation}, score-based (GES) \citep{chickering2002ges}, and continuous (NOTEARS) \citep{zheng2018notears}---can provide partial graphs; CausalWrap uses such knowledge to constrain a generator's output distribution. Kernel independence measures (Hilbert-Schmidt independence criterion, HSIC \citep{gretton2005hsic}, distance correlation \citep{szekely2007dcor}, KCIT \citep{zhang2012kcit}) serve as differentiable penalties, and penalty/augmented-Lagrangian methods \citep{bertsekas1999nonlinear} provide the optimization framework. Classifier two-sample tests \citep{lopez2016classifier2st} offer practical alternatives for distributional comparison.

\section{Problem Setup}
\label{sec:setup}

\subsection{Data and Base Generator}
Let $X=(X_1,\dots,X_d)\in\X\subseteq\R^d$ denote a mixed-type tabular record with $d$ columns (continuous or binary; multi-class categorical columns can be one-hot encoded into binary indicators). Let $P$ denote the (unknown) population distribution over $\X$, and let $\D_n=\{x^{(i)}\}_{i=1}^n\sim P^n$ be the observed dataset.

A \emph{base} synthetic data generator is a (possibly neural) sampling mechanism $G: Z\to\X$ with latent noise $Z\sim\nu$ (e.g., GAN generator, VAE decoder, or diffusion reverse process). It induces a model distribution $Q_{\mathrm{base}}$ on $\X$. We assume only black-box access to the base model: we can (i)~train it on $\D_n$ using its standard procedure (or receive a pretrained model) and (ii)~draw samples from $Q_{\mathrm{base}}$. CausalWrap does not require access to the base generator's internal parameters, training loss, or gradients (see Figure~\ref{fig:architecture} for an overview of the pipeline).

\subsection{Partial Causal Knowledge}
We represent available causal knowledge (illustrated in Figure~\ref{fig:partial_knowledge}) by a tuple $\Kcal = (E^+, E^0, \mathcal{M})$, where:
\begin{itemize}
    \item $E^+\subseteq[d]\times[d]$: \emph{trusted} directed edges $i\to j$ (``$X_i$ is a direct cause of $X_j$'').
    \item $E^0\subseteq[d]\times[d]$: \emph{forbidden} edges (``$X_i$ does not directly cause $X_j$'').
    \item $\mathcal{M}$: qualitative constraints (monotonicity, sign restrictions). Each $m\in\mathcal{M}$ specifies a 4-tuple $(i,j,S,\sigma)$ meaning ``$X_j$ is $\sigma$-monotone in $X_i$ given $X_S$'', with $\sigma\in\{+,-\}$. The conditioning set $S$ is arbitrary and need not correspond to a parent set; the formalism encodes any qualitative monotonic relationship, though in practice the constraints are typically motivated by causal reasoning.
\end{itemize}
This knowledge may be elicited from domain experts, derived from temporal structure (e.g., pre-treatment variables precede treatment), or extracted from data via causal discovery algorithms that return partial graphs or equivalence classes \citep{spirtes2000causation,zheng2018notears}. Temporal precedence information can be encoded by adding forbidden edges to $E^0$ (e.g., forbidding post-treatment variables from directly causing pre-treatment variables).

\begin{example}[ICU partial graph]
In an ICU cohort, an expert might specify: $E^+=\{$severity$\to$treatment, treatment$\to$outcome$\}$, $E^0=\{$treatment$\to$age, treatment$\to$sex$\}$, and $\mathcal{M}=\{$severity monotonically increases outcome risk$\}$.
\end{example}

\begin{figure}[t]
\centering
\begin{tikzpicture}[
  node distance=1.0cm,
  rv/.style={circle, draw, thick, minimum size=7mm, fill=lightblue, font=\small},
  known/.style={-{Stealth[length=4pt]}, thick, protectgreen},
  forbidden/.style={-, thick, riskred, dashed},
]
  \node[rv] (X1) {$X_1$};
  \node[rv, right=of X1] (X2) {$X_2$};
  \node[rv, below=of X1] (X3) {$X_3$};
  \node[rv, below=of X2] (X4) {$X_4$};

  \draw[known] (X1) -- (X3)
    node[midway, left, font=\scriptsize, protectgreen] {$E^+$};
  \draw[known] (X2) -- (X4)
    node[midway, right, font=\scriptsize, protectgreen] {$E^+$};
  \draw[forbidden] (X1) -- (X4)
    node[midway, above, font=\scriptsize, riskred, sloped] {$E^0$};
  \draw[-, thick, gray, densely dotted] (X1) -- (X2)
    node[midway, above, font=\scriptsize, gray] {?};
\end{tikzpicture}
\caption{Partial causal knowledge $\Kcal=(E^+,E^0,\mathcal{M})$.
{\color{protectgreen}Green solid}: trusted edges ($E^+$).
{\color{riskred}Red dashed}: forbidden edges ($E^0$).
Gray dotted: unknown status.}
\label{fig:partial_knowledge}
\end{figure}
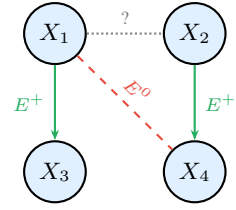

\subsection{Utility and Constraint Functionals}
Let $\mathcal{U}(P,Q)$ be a distributional loss measuring conventional similarity (e.g., IPM such as Wasserstein $\Wass$ or MMD, or adversarial loss) and let $\Omega_{\Kcal}(Q)$ be a nonnegative functional that quantifies violation of knowledge constraints. We aim to learn a synthetic distribution that trades off conventional resemblance and causal constraint satisfaction:
\begin{equation}
\label{eq:population_obj}
\min_{\phi}\; \mathcal{U}(P,Q_\phi) + \lambda\, \Omega_{\Kcal}(Q_\phi), \qquad \lambda>0.
\end{equation}
In our setting, $Q_\phi := (f_\phi)_\#\, Q_{\mathrm{base}}$ is obtained by applying a differentiable correction map $f_\phi:\X\to\X$ to samples from a pretrained base generator $Q_{\mathrm{base}}$.

\paragraph{Constraint Functionals.} We define $\Omega_{\Kcal}(Q) = \alpha\, \Omega_{\mathrm{CI}}(Q;E^+,E^0) + \beta\,\Omega_{\mathrm{mono}}(Q;\mathcal{M})$ as a weighted sum of two types of penalties:

\paragraph{(1) Residualized independence constraints.} Let $\mathrm{Pa}^+(j)$ denote the trusted parents of $X_j$ induced by $E^+$. We fit simple edge models $\hat m_j(X_{\mathrm{Pa}^+(j)})$ on real data (linear regression for continuous targets, logistic regression for binary targets) and define residuals $\tilde X_j := X_j - \hat m_j(X_{\mathrm{Pa}^+(j)})$ (with $\tilde X_j=X_j$ when $\mathrm{Pa}^+(j)=\emptyset$). For each forbidden pair $(a,b)\in E^0$, we penalize dependence between residuals under $Q$:
\begin{equation}
\Omega_{\mathrm{CI}}(Q) := \sum_{(a,b)\in E^0} w_{a,b}\,\mathrm{HSIC}(\tilde X_a,\tilde X_b; Q),
\end{equation}
where HSIC is computed with a radial basis function (RBF) kernel \citep{gretton2005hsic}. This residualization acts as a differentiable proxy for conditional-independence structure implied by the trusted graph (see Figure~\ref{fig:hsic_pipeline}). Although HSIC is symmetric, the residualization step is direction-aware: for a forbidden edge $a\to b$, the residual $\tilde X_b$ is computed by regressing out the trusted parents $\mathrm{Pa}^+(b)$, so the penalty structure reflects the directed nature of $E^0$. Note that it is a \emph{proxy}, not an exact conditional independence test: if a true parent $X_c \in \mathrm{Pa}(X_b) \setminus \mathrm{Pa}^+(b)$ is missing from the trusted set, its effect remains in the residual $\tilde X_b$. Any correlation of $X_a$ with $X_c$ via another path then inflates $\mathrm{HSIC}(\tilde X_a, \tilde X_b)$ even though $a \to b$ is truly absent---a false penalty. In practice this is mitigated when $E^+$ captures the important parents.

\begin{figure}[t]
\centering
\begin{tikzpicture}[
  box/.style={draw, rounded corners, minimum height=0.65cm,
              minimum width=1.4cm, align=center, font=\scriptsize, thick},
  fwd/.style={-Stealth, thick},
  bwd/.style={Stealth-, thick, riskred}
]
  \node[box, fill=white] (base) {$\tilde x \sim Q_{\mathrm{base}}$};
  \node[box, fill=white, right=0.7cm of base] (corr) {$x = f_\phi(\tilde x)$};
  \node[box, fill=white, right=0.7cm of corr] (resid) {Residuals\\[-1pt]$\tilde X_a,\, \tilde X_b$};
  \node[box, fill=white, right=0.7cm of resid] (hsic) {HSIC\\[-1pt]penalty};

  \draw[fwd] (base) -- (corr);
  \draw[fwd] (corr) -- node[below, font=\tiny]{$-\hat m$} (resid);
  \draw[fwd] (resid) -- (hsic);

  \draw[bwd, bend left=30] (corr.north) to
    node[above, font=\tiny, riskred]{$\nabla_\phi$} (hsic.north);
\end{tikzpicture}
\caption{Gradient flow through the CI penalty. \textbf{Forward} (black): raw samples pass through $f_\phi$; frozen edge models $\hat{m}$ produce residuals; HSIC on residual pairs gives the penalty. \textbf{Backward} (red): gradients flow back through $f_\phi$ only (edge models are frozen).}
\label{fig:hsic_pipeline}
\end{figure}
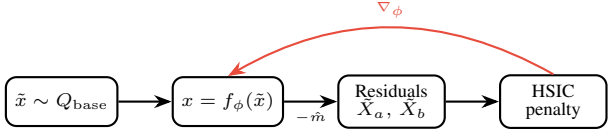

\paragraph{(2) Monotonic/sign constraints.} For a constraint $(i,j,S,+)\in\mathcal{M}$ (``$X_j$ nondecreasing in $X_i$ given $X_S$''), define a penalty via paired synthetic samples:
\begin{multline}
\Omega_{\mathrm{mono}}(Q) := \sum_{(i,j,S,\sigma)\in\mathcal{M}} \\
  \E_{\tilde x\sim Q}\Big[\max\big(0,\; -\sigma\cdot \hat\partial_{x_i}\E_Q[X_j\mid X_S{=}\tilde x_S, X_i{=}\tilde x_i]\big)\Big],
\end{multline}
where $\hat\partial_{x_i}$ is a finite-difference estimate of the conditional effect of $X_i$ on $X_j$ given $X_S$. Concretely, for each base sample $\tilde x\sim Q_{\mathrm{base}}$ we construct a twin $\tilde x'$ identical to $\tilde x$ except that $\tilde x'_i = \tilde x_i + \delta$ for a small perturbation $\delta>0$. Both $\tilde x$ and $\tilde x'$ are passed through the correction map, yielding $x=f_\phi(\tilde x)$ and $x'=f_\phi(\tilde x')$. A hinge loss penalizes sign violations: $\max(0, -\sigma\cdot(x'_j - x_j))$. The perturbation magnitude $\delta$ is a hyperparameter (we use $\delta=0.5$ in all experiments). Because $f_\phi$ is differentiable, gradients flow from the hinge loss back through both forward passes to~$\phi$.

\section{CausalWrap: A Model-Agnostic Post-Hoc Correction Wrapper}
\label{sec:method}

\begin{figure}[t]
\centering
\begin{tikzpicture}[
  box/.style={draw, rounded corners, minimum height=0.7cm,
              minimum width=1.4cm, align=center, font=\scriptsize, thick},
  fwd/.style={-Stealth, thick},
  bwd/.style={Stealth-, thick, riskred},
  node distance=0.7cm
]
  \node[box] (real) {Real Data\\$\D \sim P$};
  \node[box, right=of real] (base) {Base Gen.\\$Q_{\mathrm{base}}$};
  \node[box, right=of base] (raw) {Raw\\$\tilde x$};
  \node[box, right=of raw] (fmap) {Correction\\$f_\phi$};

  \draw[fwd] (real) -- node[above, font=\tiny]{train} (base);
  \draw[fwd] (base) -- node[above, font=\tiny]{sample} (raw);
  \draw[fwd] (raw) -- (fmap);

  \node[box, below=0.8cm of base, xshift=0.35cm] (loss) {Loss\\$\mathcal{L}(\phi)$};
  \node[box, right=0.5cm of loss] (know) {Knowledge\\$\Kcal$};
  \node[box, below=0.8cm of fmap] (corrected) {Corrected\\$x\!=\!f_\phi(\tilde x)$};

  \draw[fwd] (fmap) -- (corrected);

  \draw[fwd] (know) -- (loss);

  \draw[fwd] (corrected.south) -- ++(0,-0.3) -| (loss.south);

  \draw[fwd] (real.south) |- (loss.west);

  \draw[bwd] (fmap.south west) -- node[left, font=\tiny, riskred]{$\nabla_\phi$}
    ([xshift=-0.15cm]loss.north);
\end{tikzpicture}
\caption{CausalWrap pipeline.  A pretrained base generator (frozen) produces raw samples; a correction map $f_\phi$ produces corrected samples, which---together with real data and partial knowledge $\Kcal$---feed the loss.  \textbf{Black}: forward data flow; \textbf{\color{riskred}red}: gradient (only $\phi$ is updated).}
\label{fig:architecture}
\end{figure}
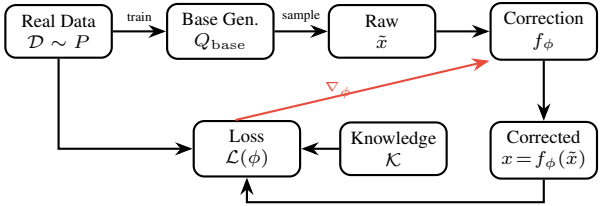

\subsection{Empirical Objective}
Given data $\D_n$ and base synthetic samples $\tilde x^{(1:m)}\sim Q_{\mathrm{base}}$, we learn a correction map $f_\phi$ and optimize an empirical analogue of \eqref{eq:population_obj} on the corrected samples $f_\phi(\tilde x)$:
\begin{equation}
\label{eq:empirical_obj}
\hat\phi \in \argmin_{\phi}\; \hat{\mathcal{U}}(\D_n, Q_\phi) + \lambda\, \hat\Omega_{\Kcal}(Q_\phi; m),
\end{equation}
where $Q_\phi := (f_\phi)_\# Q_{\mathrm{base}}$.

\subsection{Augmented Lagrangian Schedule}
\label{sec:alm}
A fixed $\lambda$ requires manual tuning and can lead to under- or over-penalization. We adopt an augmented Lagrangian method (ALM) \citep{bertsekas1999nonlinear,chatzos2023stochastic_alm} that adaptively adjusts both $\lambda$ and introduces dual variables $\mu$ for each constraint group:
{\small
\begin{multline*}
\mathcal{L}(\phi,\mu,\lambda) = \hat{\mathcal{U}}(\D_n, Q_\phi) + \sum_c \mu_c\, \hat\Omega_c(Q_\phi)
  + \tfrac{\lambda}{2}\sum_c \hat\Omega_c(Q_\phi)^2,
\end{multline*}
}
where $c$ indexes individual constraint terms. The full procedure---alternating primal gradient steps, dual variable updates, and penalty increases---is detailed in Algorithm~\ref{alg:causalwrap}; the empirical objective~\eqref{eq:empirical_obj} is minimized in each inner loop.

\begin{algorithm}[t]
\caption{CausalWrap (post-hoc correction around a base generator)}
\label{alg:causalwrap}
\begin{algorithmic}[1]
\REQUIRE data $\D_n$, pretrained base generator $Q_{\mathrm{base}}$, correction map parameters $\phi$, causal knowledge $\Kcal$, initial $\lambda_0$, growth rate $\rho$, inner steps $T_{\mathrm{inner}}$
\STATE Initialize dual variables $\mu_c\leftarrow 0$ for each constraint $c$
\FOR{outer iterations $k=1,2,\dots,K$}
  \FOR{inner steps $t=1,\dots,T_{\mathrm{inner}}$}
    \STATE Sample minibatch $\{x^{(i)}\}$ from $\D_n$ and base minibatch $\{\tilde x^{(j)}\}$ from $Q_{\mathrm{base}}$
    \STATE Apply correction: $x^{(j)} \leftarrow f_\phi(\tilde x^{(j)})$
\STATE Compute a utility surrogate $\hat{\mathcal{U}}(\D_n, \{x^{(j)}\})$ (e.g., moment matching + identity regularization)
    \STATE Compute constraint losses on corrected samples: $\hat\Omega_{\mathrm{CI}},\hat\Omega_{\mathrm{mono}}$
    \STATE Set ALM loss: $\hat L = \hat{\mathcal{U}} + \sum_c \mu_c \hat\Omega_c + \frac{\lambda}{2}\sum_c \hat\Omega_c^2$
    \STATE Update $\phi$ by stochastic gradient step on $\hat L$
  \ENDFOR
  \STATE Update duals: $\mu_c \leftarrow \mu_c + \lambda\,\hat\Omega_c(Q_\phi)$ for each $c$
  \STATE Increase penalty: $\lambda \leftarrow \rho\,\lambda$
\ENDFOR
\RETURN corrected generator $Q_{\hat\phi} := (f_{\hat\phi})_\# Q_{\mathrm{base}}$
\end{algorithmic}
\end{algorithm}

Implementation details---including the correction map architecture, HSIC estimation, monotonicity estimation, binary column handling, and computational overhead---are provided in Appendix~\ref{app:impl_details}.

\section{Theory}
\label{sec:theory}

This section formalizes what the CausalWrap wrapper can guarantee. All proofs are deferred to Appendix~\ref{app:proofs}.

\subsection{Penalty-Method Convergence}
Define the feasible set $\mathcal{F} := \{Q\in\mathcal{Q} : \Omega_{\Kcal}(Q)=0\}$ and consider the constrained population problem
\begin{equation}
\label{eq:constrained}
\min_{\phi}\; \mathcal{U}(P,Q_\phi) \quad \text{s.t.}\quad \Omega_{\Kcal}(Q_\phi)=0.
\end{equation}

\begin{assumption}[Basic regularity]
\label{assump:regularity}
(i)~The parameter space $\Phi$ is compact. (ii)~$\phi\mapsto \mathcal{U}(P,Q_\phi)$ and $\phi\mapsto \Omega_{\Kcal}(Q_\phi)$ are lower semicontinuous. (iii)~There exists at least one feasible parameter $\phi\in\Phi$ with $\Omega_{\Kcal}(Q_\phi)=0$. (iv)~$\inf_{\phi\in\Phi}\mathcal{U}(P,Q_\phi) > -\infty$.
\end{assumption}

\begin{theorem}[Penalty convergence]
\label{thm:penalty}
Under Assumption~\ref{assump:regularity}, let $\hat\phi_\lambda$ be any minimizer of the penalized population objective $\mathcal{U}(P,Q_\phi)+\lambda\Omega_{\Kcal}(Q_\phi)$. As $\lambda\to\infty$, every limit point $\phi^\star$ of $\{\hat\phi_\lambda\}$ is feasible and solves the constrained problem~\eqref{eq:constrained} (i.e., achieves minimal $\mathcal{U}$ among feasible points).
\end{theorem}

\begin{remark}[ALM convergence rate]
Under standard constraint qualification (linear independence of active constraint gradients at a local minimizer), the augmented Lagrangian method converges at a linear rate in the outer loop, and does not require $\lambda\to\infty$---a bounded $\lambda$ suffices once the dual variables $\mu$ track the true Lagrange multipliers \citep{bertsekas1999nonlinear}. This avoids the ill-conditioning that pure penalty methods suffer from.
\end{remark}

\subsection{Approximate Factorization Implies Joint Closeness}
A key use of causal knowledge is to enforce approximate Markov factorization.

\begin{proposition}[Chain-rule total variation (TV) bound for approximate conditionals]
\label{prop:chain_tv}
Let $P$ and $Q$ admit factorizations with a common topological order $\pi$:
$P(x)=\prod_{j=1}^d P_j(x_{\pi(j)}\mid x_{\pi(<j)})$ and $Q(x)=\prod_{j=1}^d Q_j(x_{\pi(j)}\mid x_{\pi(<j)})$. If for all $j$,
\begin{equation}
\E_{X_{\pi(<j)}\sim P}\big[\TV(P_j(\cdot\mid X_{\pi(<j)}),Q_j(\cdot\mid X_{\pi(<j)}))\big] \le \varepsilon_j,
\end{equation}
then $\TV(P,Q) \le \sum_{j=1}^d \varepsilon_j$.
\end{proposition}

\begin{remark}[Design rationale and gap]
\label{rem:gap}
Theorem~\ref{thm:penalty} guarantees that the causal penalties vanish as $\lambda\to\infty$, and Proposition~\ref{prop:chain_tv} guarantees that if each conditional is $\varepsilon_j$-close then the joint TV is bounded. However, there is no end-to-end theorem establishing that vanishing HSIC-based penalties imply that $Q$'s conditionals match $P$'s---the residualized HSIC is a proxy for conditional independence, not an exact test (see the false-penalty caveat in Section~\ref{sec:setup}). CausalWrap's theory therefore provides a \emph{design rationale} rather than a closed-form guarantee: the penalties are motivated by the chain-rule bound, and empirical results validate that they improve causal fidelity in practice (see Figure~\ref{fig:theory_connection} in the Supplement).
\end{remark}

\section{Experiments}
\label{sec:experiments}

We evaluate CausalWrap across three tiers of increasing realism. Unless otherwise noted, experiments use 5 random seeds and we report seed-averaged means.

\subsection{Experimental Setup}

\paragraph{Base generators.} We wrap three base generators spanning GAN/diffusion/VAE families: CTGAN \citep{xu2019ctgan}, TabDDPM \citep{kotelnikov2023tabddpm}, and TVAE \citep{xu2019ctgan}. Each base generator is trained using its standard implementation, and CausalWrap is then fit as a post-hoc correction module on top of base samples.

\paragraph{Oracle comparator (Tier 1 only).} For simulated SCMs (Tier 1), we include an oracle comparator that directly samples from the true SCM via ancestral sampling, serving as a full-knowledge reference distribution. Since the average treatment effect (ATE) is estimated using a simple outcome-regression estimator, oracle sampling is not always the lowest ATE error on nonlinear SCMs due to estimator misspecification.

\paragraph{Other causal-aware methods.} STEAM \citep{amad2025steam} and C-DGM \citep{stoian2024constrained} address different problem settings (treatment-mechanism losses and domain-value constraints, respectively) and are not directly comparable; a systematic cross-regime comparison is left to future work.

\paragraph{Metrics.} We evaluate on three axes:
\begin{itemize}
    \item \textbf{Conventional:} MMD (RBF kernel), column-wise Jensen--Shannon divergence (JSD), and TSTR accuracy (train-on-synthetic, test-on-real).
    \item \textbf{Structural:} conditional independence pass rate (CI Pass): the fraction of specified ``forbidden'' pairs whose absolute residual correlation falls below $0.08$ (a correlation-threshold heuristic).
    \item \textbf{Causal:} When ground truth is available (Tiers~1--2), we report ATE absolute error ($|\hat\tau_{\mathrm{syn}} - \tau_{\mathrm{true}}|$) and conditional ATE (CATE) precision in estimation of heterogeneous effects (PEHE). When ground truth is unavailable (Tier~3), we report an ATE \emph{agreement} score comparing effect estimates on synthetic vs.\ real data across an estimator ensemble (OR, IPW, AIPW, TMLE).
\end{itemize}

\subsection{Tier 1: Simulated SCMs}
\label{sec:exp_scm}

\paragraph{Setup.} We generate data from three SCM families with known ground-truth interventional distributions:
\begin{itemize}
    \item \textbf{Linear-Gaussian (LG):} $d=10$ variables, random DAG with expected degree 2, linear mechanisms with additive Gaussian noise.
    \item \textbf{Nonlinear-Additive (NLA):} Same graph structure, nonlinear additive mechanisms using $\tanh(\cdot)$ and a small quadratic term, with additive Gaussian noise.
    \item \textbf{Mixed-Type (MT):} $d=10$ variables with mixed continuous and binary columns (5 continuous, 5 binary); continuous mechanisms are nonlinear additive and binary mechanisms use a logistic link with Bernoulli noise.
\end{itemize}
For each SCM we sample $n=5{,}000$ training records and generate $m=5{,}000$ synthetic records. We set the outcome to the sink node $Y=x_d$ and choose treatment $T$ as the earliest ancestor of $Y$ (ensuring a directed causal path). Ground-truth ATE is computed by Monte Carlo ancestral sampling under interventions $\doop(T{=}1)$ and $\doop(T{=}0)$ (100k samples each). Causal knowledge $\Kcal$ is created by revealing 50\% of the true edges in $E^+$, adding forbidden edges from $E^0$ (all non-edges involving root nodes), and including 2 monotonicity constraints from $\mathcal{M}$.

\paragraph{Results.} Table~\ref{tab:tier1_all} summarizes the results (full conventional metrics in Table~\ref{tab:lg_results}, Appendix~\ref{app:tier1_full}). CW consistently improves distributional metrics (MMD/JSD). Its impact on ATE error is modest and mixed: CW improves CTGAN/TVAE on average but is close to neutral for TabDDPM, with SCM-dependent regressions on nonlinear settings. This reflects a trade-off: enforcing causal constraints helps when the base generator violates them but can introduce estimator-facing shifts in harder regimes. On NLA, the linear outcome-regression estimator can be biased even under oracle sampling, muting CW gains. On MT, results are generator-dependent ($-4.8\%$ for TabDDPM but $+23.3\%$ for CTGAN), suggesting binary columns pose additional challenges. Tier~1 effects are modest; we next evaluate richer semi-synthetic benchmarks.

\begin{table}[t]
\centering
\caption{\textbf{Tier 1: CausalWrap improvement ($\Delta$\,\%) in ATE error} across 3 SCM types (5 seeds). \impup{Green} = CW reduces error; \impdn{red} = CW increases error. \emph{Gap closed} reports the fraction of the oracle gap closed by CW: $(e_{\mathrm{base}}-e_{\mathrm{CW}})/(e_{\mathrm{base}}-e_{\mathrm{oracle}})$, shown as ``--'' when $e_{\mathrm{oracle}}\ge e_{\mathrm{base}}$ or when the oracle headroom $(e_{\mathrm{base}}-e_{\mathrm{oracle}})/e_{\mathrm{base}}$ is $<5\%$ (unstable ratio). \textbf{Overall} is the unweighted average across the 3 SCM types. Best partial-knowledge result \underline{underlined}.}
\label{tab:tier1_all}
\footnotesize
\begin{tabular}{l cccc}
\toprule
Generator & LG & NLA & MT & Overall \\
\midrule
CTGAN       & 0.832 & 0.090 & 0.070 & 0.331 \\
$\Delta$ CW & \impup{$-$9.1\%} & \impdn{$+$11.6\%} & \impdn{$+$23.3\%} & \impup{$-$4.9\%} \\[3pt]
Gap closed  & \impup{$+$29.1\%} & -- & \impdn{$-$42.3\%} & \impup{$+$23.0\%} \\
\cmidrule(lr){1-5}
TabDDPM     & 0.644 & 0.180 & 0.042 & 0.289 \\
$\Delta$ CW & \impup{$-$1.7\%} & \impdn{$+$4.6\%} & \impup{$-$4.8\%} & \impup{$-$0.6\%} \\[3pt]
Gap closed  & \impup{$+$15.4\%} & -- & \impup{$+$19.0\%} & \impup{$+$5.6\%} \\
\cmidrule(lr){1-5}
TVAE        & 0.766 & 0.130 & 0.039 & 0.312 \\
$\Delta$ CW & \impup{$-$6.3\%} & \impdn{$+$0.1\%} & \impup{$-$2.6\%} & \impup{\underline{$-$5.3\%}} \\
Gap closed  & \impup{$+$25.0\%} & -- & \impup{$+$13.5\%} & \impup{$+$32.0\%} \\
\midrule
Oracle \scriptsize{(true SCM)} & \textbf{0.572} & \textbf{0.177} & \textbf{0.032} & \textbf{0.260} \\
\bottomrule
\end{tabular}
\end{table}

\subsection{Tier 2: Semi-Synthetic Causal Benchmarks}
\label{sec:exp_semisyn}

\paragraph{Setup.} We use two standard causal inference benchmarks:
\begin{itemize}
    \item \textbf{IHDP} (Infant Health and Development Program): 747 records, 25 covariates, binary treatment, continuous outcome. We use the widely used semi-synthetic IHDP replications \citep{hill2011ihdp} packaged with CEVAE \citep{louizos2017cevae} and evaluate 10 replications due to computational cost. Ground-truth potential outcome means $(m_0,m_1)$ are provided, so the individual treatment effect (ITE) is $m_1-m_0$ and ATE is its mean.
    \item \textbf{ACIC-style suite:} A lightweight semi-synthetic benchmark inspired by ACIC \citep{dorie2019acic} dimensions ($n=4{,}800$, $d=58$, binary treatment, continuous outcome). We evaluate 10 data-generating process (DGP) settings that vary baseline outcome complexity (linear vs.\ nonlinear) and treatment effect structure (homogeneous vs.\ heterogeneous), with known ground-truth ITE $\tau(X)$ by construction.
\end{itemize}

\paragraph{Causal knowledge.} For IHDP, we use a small set of illustrative constraints in the benchmark feature space: trusted edges into treatment/outcome ($x_1\!\to$ treatment, $x_2\!\to$ outcome), forbidden edges from treatment to selected covariates (treatment$\!\to x_1$, treatment$\!\to x_3$), and one monotonic constraint ($x_2$ positively affects outcome). For the ACIC-style suite, we mark a small subset of covariates as trusted parents of treatment (e.g., $x_1,\dots,x_5\to$ treatment) and include a small number of illustrative forbidden edges.

\paragraph{Protocol.} For each benchmark setting, we train each generator on the real training data, generate $n_{\mathrm{syn}} = n_{\mathrm{real}}$ synthetic records, then estimate ATE and CATE on the synthetic data using outcome regression. We compare estimated effects to ground-truth. Unless otherwise noted we use 5 random seeds; due to computational cost, ACIC results use 2 seeds and a smaller training budget (50 iterations) for the base generators.

\paragraph{Results.} CausalWrap provides its strongest improvements on semi-synthetic benchmarks (Figures~\ref{fig:ihdp_results} and~\ref{fig:acic_results}). Across IHDP and ACIC settings, the wrapper often reduces ATE error relative to the base generator while largely preserving conventional utility, though gains are not uniform across all bases and settings. For example, on IHDP CW improves CTGAN and TabDDPM mean ATE error (about $-7\%$ and $-57\%$, respectively) while increasing TVAE error; on ACIC it substantially improves TVAE (about $-63\%$) and also improves CTGAN and TabDDPM (about $-19\%$ and $-12\%$). These patterns are consistent with the intuition behind Proposition~\ref{prop:chain_tv}: when confounding structure is richer, enforcing causal constraints can provide more leverage than on simple simulated SCMs.

\begin{figure*}[t]
    \centering
    \begin{subfigure}[b]{0.48\linewidth}
        \includegraphics[width=\linewidth]{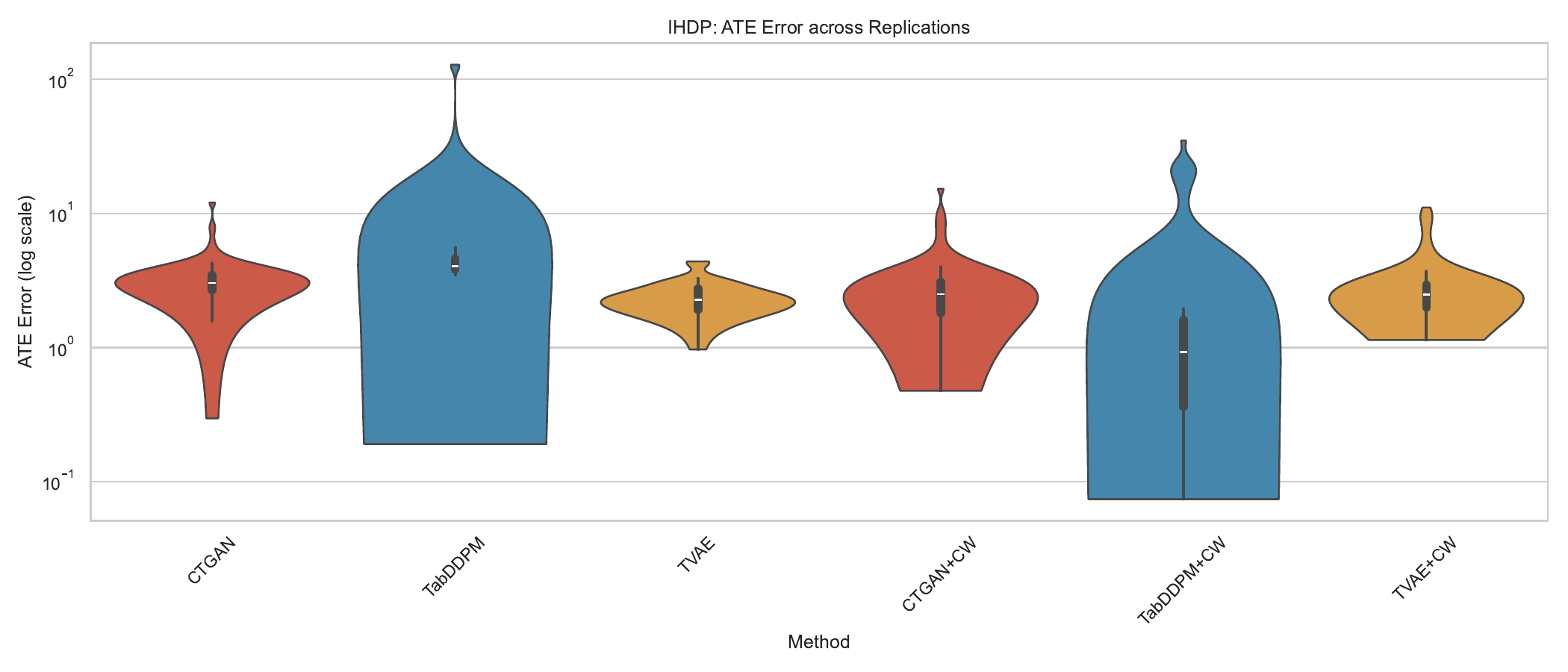}
        \caption{IHDP: ATE absolute error}
    \end{subfigure}
    \hfill
    \begin{subfigure}[b]{0.48\linewidth}
        \includegraphics[width=\linewidth]{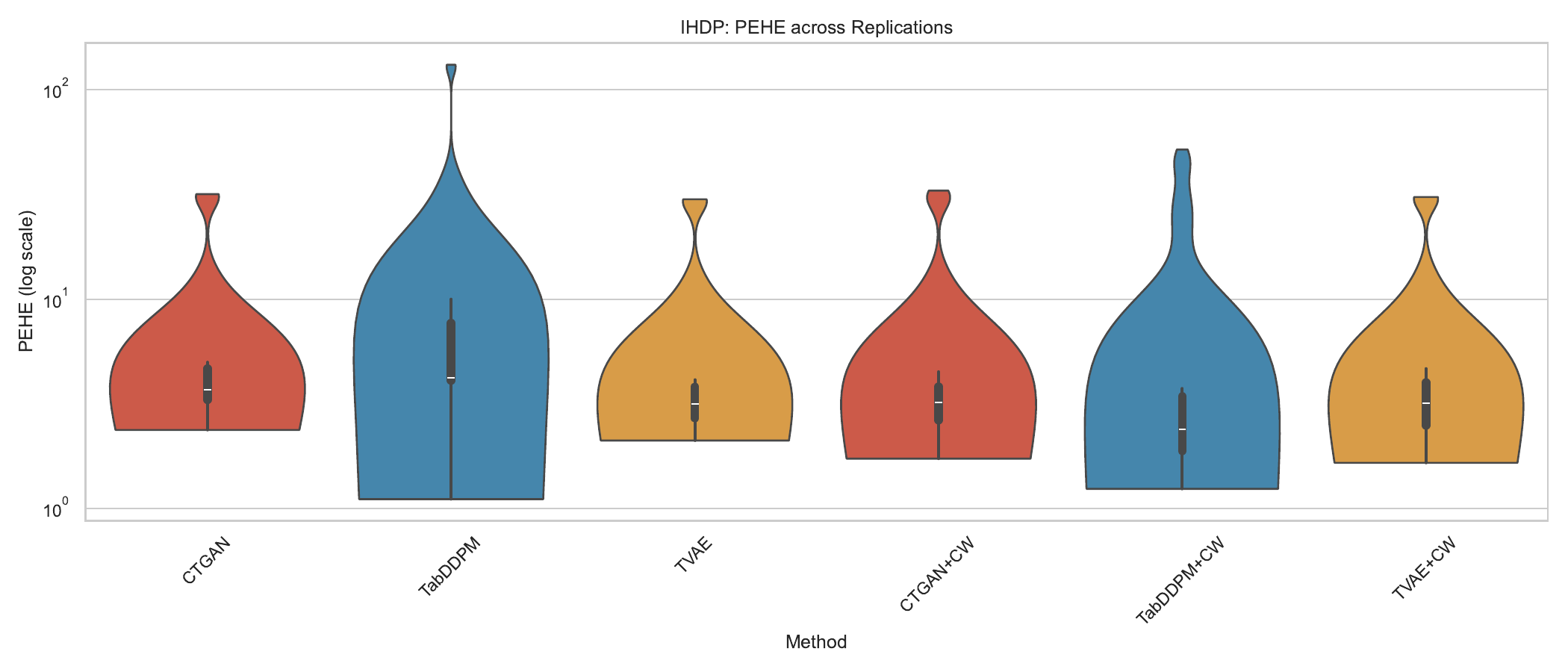}
        \caption{IHDP: CATE PEHE}
    \end{subfigure}
    \caption{\textbf{Tier 2: IHDP benchmark} (10 replications, 5 seeds; log-scale violin plots). CW reduces mean ATE error for CTGAN and TabDDPM but can increase error for TVAE, and may increase tail variance on some replications.}
    \label{fig:ihdp_results}
\end{figure*}

\begin{figure*}[t]
    \centering
    \includegraphics[width=0.9\linewidth]{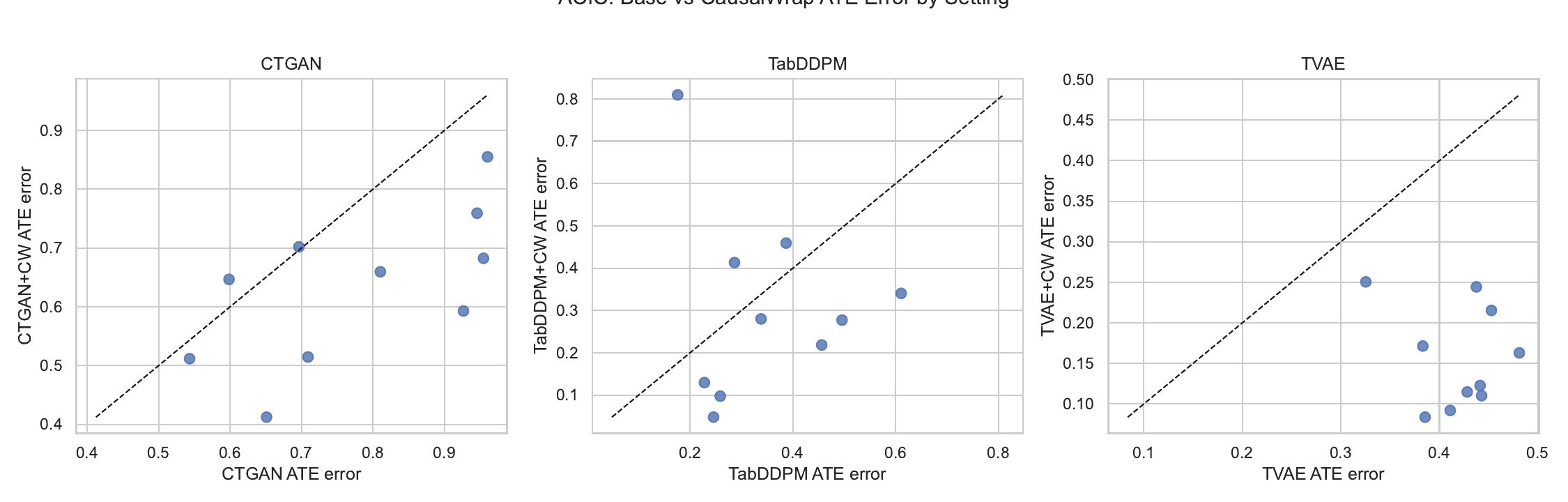}
        \caption{\textbf{Tier 2: ACIC-style ATE error across 10 DGP settings} (2 seeds). CW improves all three bases on most settings (largest gains for TVAE), illustrating base-model sensitivity of post-hoc correction.}
    \label{fig:acic_results}
\end{figure*}

\subsection{Tier 3: Real-World ICU Cohort}
\label{sec:exp_icu}

\paragraph{Setup.} We extract an adult ICU cohort from MIMIC-IV \citep{johnson2023mimiciv}, retaining the first ICU stay per subject and restricting to adults (age$\ge18$). We define a 24-hour baseline window from ICU admission. The resulting cohort ($n=2{,}000$) contains demographics (age, sex, race/ethnicity), two binary treatment indicators (any vasopressor administration and mechanical ventilation within the first 24 hours), three baseline laboratory values (lactate, creatinine, white blood cell count (WBC)---first measurement in the window, median-imputed when missing), and outcome (28-day mortality). The cohort has 29\% vasopressor-treated patients and 13\% mortality.

\paragraph{Causal knowledge.} We encode a conservative, clinician-plausible subset of structural knowledge for the 24-hour ICU baseline snapshot (Figure~\ref{fig:icu_knowledge}). We treat early lactate and creatinine as severity proxies that influence both vasopressor initiation and subsequent mortality risk, and include demographics (age, sex) as stable drivers of treatment propensity. Concretely, $E^+$ contains trusted edges from baseline covariates to vasopressor treatment and from $\{\text{age}, \text{sex}, \text{ventilation}, \text{vasopressors}, \text{lactate}, \text{creatinine}\}$ to 28-day mortality, and $\mathcal{M}$ imposes monotonic constraints that increasing lactate and creatinine should not decrease mortality risk. We additionally include temporal forbidden edges $E^0$ encoding biological and temporal impossibilities: treatment cannot cause pre-treatment demographics, and the outcome cannot retroactively cause baseline-window measurements.

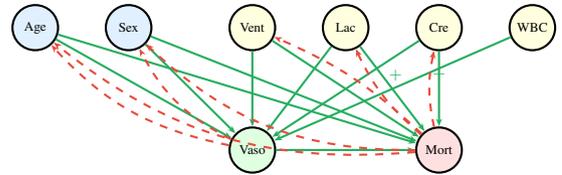
\begin{figure}[t]
\centering
\begin{tikzpicture}[
  node distance=0.7cm and 0.6cm,
  rv/.style={circle, draw, thick, minimum size=6mm, font=\tiny, inner sep=1pt},
  demo/.style={rv, fill=lightblue},
  lab/.style={rv, fill=lightyellow},
  trt/.style={rv, fill=lightgreen},
  outc/.style={rv, fill=lightred},
  ep/.style={-{Stealth[length=3pt]}, thick, protectgreen},
  forb/.style={-{Stealth[length=3pt]}, thick, riskred, dashed},
  mono/.style={-{Stealth[length=3pt]}, thick, protectgreen,
               postaction={decorate, decoration={markings,
               mark=at position 0.55 with {\node[font=\tiny, above=1pt]{$+$};}}}},
]
  \node[demo] (age) {Age};
  \node[demo, right=of age] (sex) {Sex};
  \node[lab, right=1.0cm of sex] (vent) {Vent};
  \node[lab, right=of vent] (lac) {Lac};
  \node[lab, right=of lac] (cre) {Cre};
  \node[lab, right=of cre] (wbc) {WBC};

  \node[trt, below=1.0cm of vent] (vaso) {Vaso};
  \node[outc, below=1.0cm of cre] (mort) {Mort};

  \draw[ep] (age) -- (vaso);
  \draw[ep] (sex) -- (vaso);
  \draw[ep] (vent) -- (vaso);
  \draw[ep] (lac) -- (vaso);
  \draw[ep] (cre) -- (vaso);
  \draw[ep] (wbc) -- (vaso);

  \draw[ep] (vaso) -- (mort);
  \draw[ep] (age) -- (mort);
  \draw[ep] (sex) -- (mort);
  \draw[ep] (vent) -- (mort);

  \draw[mono] (lac) -- (mort);
  \draw[mono] (cre) -- (mort);

  \draw[forb, bend left=18] (vaso) to (age);
  \draw[forb, bend left=18] (vaso) to (sex);
  \draw[forb, bend left=12] (mort) to (lac);
  \draw[forb, bend left=12] (mort) to (cre);
  \draw[forb, bend right=10] (mort) to (vent);
  \draw[forb, bend left=22] (mort) to (age);
  \draw[forb, bend left=22] (mort) to (sex);
\end{tikzpicture}
\caption{Tier~3 ICU causal knowledge $\Kcal$.
{\color{protectgreen}Green solid}: trusted edges $E^+$ (demographics, labs, and ventilation cause vasopressor use; treatment and baseline variables cause mortality).
{\color{protectgreen}$+$} on an arrow: monotonic constraint in $\mathcal{M}$ (lactate$\uparrow$, creatinine$\uparrow$ $\Rightarrow$ mortality risk$\uparrow$).
{\color{riskred}Red dashed}: forbidden directions $E^0$ (temporal/biological impossibilities).}
\label{fig:icu_knowledge}
\end{figure}

\paragraph{Evaluation.} Since ground-truth treatment effects are unknown, we evaluate \emph{agreement} of causal conclusions between real and synthetic data. We compute a reference ATE vector $\{\hat\tau_l^{(\mathrm{real})}\}_{l=1}^L$ on the real cohort using an estimator ensemble of $L=4$ estimators (outcome regression, IPW, AIPW, TMLE) for the vasopressor treatment indicator, compute the analogous vector on synthetic data, and report
$1 - \tfrac{1}{L}\sum_{l=1}^L |\hat\tau_l^{(\mathrm{syn})}-\hat\tau_l^{(\mathrm{real})}|\big/\bigl(\tfrac{1}{L}\sum_{l} |\hat\tau_l^{(\mathrm{real})}|\bigr)$ (clipped to $[0,1]$). The estimator ensemble reduces sensitivity to any single misspecified estimator. Reproducing Tier~3 requires credentialed access to MIMIC-IV via PhysioNet \citep{goldberger2000physionet}; the cohort extraction and evaluation pipeline will be released with our code upon acceptance.

\paragraph{Results.} Figure~\ref{fig:icu_radar} reports Tier~3 results. CausalWrap lifts ATE agreement for the two generators whose baselines produce zero or near-zero agreement: TabDDPM ($0.00\to0.38$, the largest absolute gain) and TVAE ($0.15\to0.28$). For CTGAN, whose uncorrected samples already achieve moderate agreement ($0.56$), the wrapper's additional constraints do not improve---and slightly reduce---causal fidelity. On distributional metrics, CW substantially tightens CTGAN's MMD ($-64\%$) while leaving TSTR essentially unchanged, but for TabDDPM the correction widens the distributional gap (MMD $+61\%$, TSTR $-28\%$), consistent with the Tier~1--2 finding that post-hoc correction faces a steeper fidelity--constraint trade-off when the base generator poorly matches the joint distribution. TVAE shows a non-trivial MMD increase ($+40\%$) alongside its ATE gain. Overall, CW's benefit is strongest when the base generator clearly violates the causal structure in $\Kcal$; when the base already approximately respects the encoded constraints, the optimization budget spent on $E^0$ and $\mathcal{M}$ penalties yields diminishing returns.

\begin{figure}[t]
    \centering
    \includegraphics[width=0.55\linewidth]{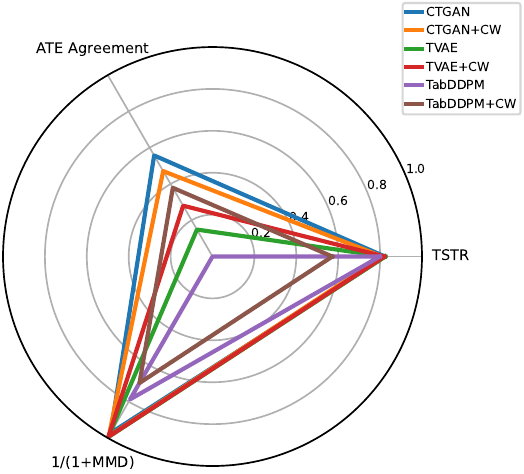}
    \caption{\textbf{Tier 3: ICU cohort (MIMIC-IV).} Radar plot of seed-averaged metrics (TSTR, ATE agreement, and $1/(1+\mathrm{MMD})$) for each base generator and its CausalWrap variant. All axes are scaled so larger is better.}
    \label{fig:icu_radar} 
\end{figure}

\vspace{-1em}
\subsection{Ablation Studies}
\label{sec:ablations}

We ablate key design choices using the Linear-Gaussian SCM with TabDDPM as the base generator.

\paragraph{Constraint type ablation.} We compare using only CI constraints ($\alpha>0, \beta=0$), only monotonicity ($\beta>0, \alpha=0$), and both combined (Figure~\ref{fig:ablation_constraints}, Supplement). The constraint types contribute differently: monotonicity constraints provide the strongest signal in this setting, while CI constraints can be weaker and may introduce trade-offs when combined.

\paragraph{Knowledge fraction.} We vary the fraction of correct edges revealed in $E^+$ from 0\% to 100\% (Figure~\ref{fig:ablation_knowledge}, Supplement). ATE error exhibits a non-monotonic pattern, suggesting that over-constraining the generator can be counterproductive when constraints conflict or over-regularize.

\paragraph{Robustness to wrong edges.} We corrupt $E^+$ by replacing a fraction of edges with incorrect ones (Figure~\ref{fig:ablation_wrong}, Appendix~\ref{app:ablations}). ATE error remains stable under moderate corruption (${\le}30\%$ wrong edges) without catastrophic degradation.

\paragraph{ALM vs.\ fixed penalty.} The ALM schedule outperforms fixed-$\lambda$ variants (Figure~\ref{fig:ablation_alm}, Appendix~\ref{app:ablations}), suggesting that adaptive dual updates reduce sensitivity to penalty tuning.

\paragraph{Forbidden-edge ablation ($E^0$).} Adding temporal forbidden edges to $\Kcal$ improves ATE agreement for TabDDPM but weakens gains for CTGAN and TVAE, likely due to \emph{constraint budget dilution}: HSIC penalties for edges the base generator already approximately satisfies consume optimization capacity without useful gradient signal. Full results are in Table~\ref{tab:e0_ablation} (Supplement).


\section{Discussion and Limitations}
\label{sec:discussion}

We summarize the main takeaways and acknowledge current limitations.

\paragraph{Theoretical contributions and limitations thereof.} CausalWrap provides a principled framework for injecting partial causal knowledge into black-box generators. Theorem~\ref{thm:penalty} guarantees penalty convergence and Proposition~\ref{prop:chain_tv} relates approximate conditional matching to joint TV control. As discussed in Remark~\ref{rem:gap}, these two results do not compose into a single end-to-end guarantee because residualized HSIC is a proxy for conditional independence, not an exact test. The theory therefore serves as a \emph{design rationale}---motivating the penalty structure---rather than a closed-form bound on $\TV(P,Q)$.

\paragraph{Key empirical findings.} Three patterns emerge. \emph{(1)~CW's benefit scales with confounding complexity}: gains are modest on simple simulated SCMs but larger on semi-synthetic benchmarks with richer confounding, consistent with Proposition~\ref{prop:chain_tv}. \emph{(2)~Constraint types contribute distinctly}: monotonicity constraints are often the most effective single type; CI constraints can be weaker and may introduce trade-offs when combined, suggesting practitioners should prioritize mechanism-level constraints when available. \emph{(3)~Knowledge quality matters non-monotonically}: over-constraining can hurt, wrong-edge corruption shows moderate robustness, and the forbidden-edge ablation (Table~\ref{tab:e0_ablation}, Supplement) confirms that redundant $E^0$ constraints can dilute the optimization budget. CW is best deployed with a conservative subset of high-confidence, \emph{actively violated} constraints.

\paragraph{Limitations.}
CW can only enforce differentiable penalty functionals; hard combinatorial constraints are not supported. Over-specification can dilute the optimization budget (as shown by the knowledge-fraction and $E^0$ ablations). The residualized HSIC penalties are proxies for CI with limited statistical power under edge-model misspecification---a likely contributor to weaker CI-only results. As a post-hoc surrogate-utility correction, CW may not preserve higher-order interactions beyond the explicitly enforced constraints. Finally, Tier~3 relies on restricted-access MIMIC-IV data, limiting reproducibility.

\paragraph{Future directions.}
Natural extensions include longitudinal data (temporal partial orders as forbidden edges), Bayesian graph posterior sampling to weight constraints by posterior probability, and adaptive constraint selection to identify which penalties provide the most gradient signal for a given base generator.

\vspace{-1em}
\section{Conclusion}
\label{sec:conclusion}

CausalWrap provides a model-agnostic framework for injecting partial causal knowledge into black-box tabular generators via a lightweight post-hoc correction map. Across simulated SCMs, semi-synthetic benchmarks, and a real-world ICU cohort, CW improves causal fidelity---most strongly when the base generator clearly violates the encoded structure---while largely preserving conventional utility, with the strongest gains when constraints target actively violated causal relationships.

\bibliography{paper1}

\newpage
\onecolumn
\title{CausalWrap: Model-Agnostic Causal Constraint Wrappers\\for Tabular Synthetic Data\\(Supplementary Material)}
\maketitle

\appendix

\section{Proofs}
\label{app:proofs}

\begin{proof}[Proof of Theorem~\ref{thm:penalty}]
Let $\phi_f$ be a feasible parameter. For any $\lambda>0$,
\begin{multline*}
\mathcal{U}(P,Q_{\hat\phi_\lambda})+\lambda\Omega_{\Kcal}(Q_{\hat\phi_\lambda}) 
  \le \mathcal{U}(P,Q_{\phi_f})+\lambda\Omega_{\Kcal}(Q_{\phi_f}) = \mathcal{U}(P,Q_{\phi_f}).
\end{multline*}
Since $\mathcal{U}(P,Q_{\hat\phi_\lambda})\ge \inf_{\phi}\mathcal{U}(P,Q_\phi) > -\infty$, we obtain
$\lambda\Omega_{\Kcal}(Q_{\hat\phi_\lambda})\le \mathcal{U}(P,Q_{\phi_f})-\inf_{\phi\in\Phi}\mathcal{U}(P,Q_\phi) =: B < \infty$,
so $\Omega_{\Kcal}(Q_{\hat\phi_\lambda})\le B/\lambda \to 0$ as $\lambda\to\infty$.

By compactness of $\Phi$, take a convergent subsequence $\hat\phi_{\lambda_k}\to\phi^\star$. Lower semicontinuity gives
$\Omega_{\Kcal}(Q_{\phi^\star})\le \liminf_k \Omega_{\Kcal}(Q_{\hat\phi_{\lambda_k}})=0$, so $\phi^\star$ is feasible.

To show optimality among feasible points, fix any feasible $\phi$. From optimality of $\hat\phi_{\lambda}$,
$\mathcal{U}(P,Q_{\hat\phi_{\lambda}}) \le \mathcal{U}(P,Q_{\hat\phi_{\lambda}})+\lambda\Omega_{\Kcal}(Q_{\hat\phi_{\lambda}}) \le \mathcal{U}(P,Q_{\phi})$.
Taking $\liminf$ along the subsequence and using lower semicontinuity yields
$\mathcal{U}(P,Q_{\phi^\star})\le \mathcal{U}(P,Q_{\phi})$ for all feasible $\phi$, proving $\phi^\star$ solves~\eqref{eq:constrained}.
\end{proof}

\begin{proof}[Proof of Proposition~\ref{prop:chain_tv}]
We construct a coupling between $P$ and $Q$ via sequential sampling. Define intermediate distributions $R_k$ for $k=0,\dots,d$ where $R_0 = P$, $R_d = Q$, and $R_k$ uses $P$'s conditionals for variables $\pi(1),\dots,\pi(d-k)$ and $Q$'s conditionals for $\pi(d-k+1),\dots,\pi(d)$.

For each consecutive pair $(R_{k-1}, R_k)$, these differ only in the conditional for variable $\pi(d-k+1)$. Construct a maximal coupling at step $\pi(d-k+1)$: conditional on the history $X_{\pi(<d-k+1)}$ (which is identically distributed under both $R_{k-1}$ and $R_k$ when conditioned on identical past draws), the probability of a mismatch at this step equals $\TV(P_{\pi(d-k+1)}(\cdot\mid X_{\pi(<d-k+1)}), Q_{\pi(d-k+1)}(\cdot\mid X_{\pi(<d-k+1)}))$.

Taking expectations over the history and summing over all $d$ steps:
$\TV(P,Q) \le \TV(R_0, R_1) + \cdots + \TV(R_{d-1}, R_d) \le \sum_{j=1}^d \varepsilon_j$.
This uses the triangle inequality for TV and the fact that each intermediate TV is bounded by $\varepsilon_j$ via the conditional coupling.
\end{proof}

\section{Implementation Details}
\label{app:impl_details}

\paragraph{Correction map and utility surrogate.} We implement $f_\phi$ as a small residual MLP initialized near the identity map. We use a lightweight surrogate utility consisting of moment matching (means and covariances between corrected synthetic and real minibatches) plus an identity regularizer $\|f_\phi(\tilde x)-\tilde x\|_2^2$ to prevent over-correction; this regularizer is set stronger for TabDDPM than for CTGAN/TVAE in our experiments.

\paragraph{Dependence estimation.} We use the biased HSIC estimator with a Gaussian RBF kernel whose bandwidth is set by the median heuristic \citep{gretton2005hsic}. In our implementation, we (i)~fit edge models $\hat m_j$ on real data using trusted parents $E^+$, (ii)~compute residuals $\tilde X_j$ on corrected samples, and (iii)~apply HSIC to residual pairs specified by $E^0$ (subsampling pairs for efficiency).

\paragraph{Monotonicity estimation.} We estimate monotonicity using paired synthetic samples: we perturb a specified feature by a small $\delta$ and apply a hinge penalty to sign violations in the corresponding output change.

\paragraph{Binary columns.} For columns detected as binary $\{0,1\}$, we clamp values to $[0,1]$ when computing constraint penalties during training. At sampling time, we interpret corrected values as Bernoulli probabilities (with a simple marginal mean-shift calibration) and sample binary values; if a column is already binary and close to the target marginal, we leave it unchanged to avoid injecting additional Monte Carlo noise.

\paragraph{Computational overhead.} The constraint penalties add $O(m\cdot|\mathcal{C}|\cdot d_S)$ computation per minibatch for CI constraints (where $d_S$ is the maximum conditioning set size) and $O(m\cdot|\mathcal{M}|)$ for monotonicity. In our settings, the additional overhead is typically modest relative to base training.

\paragraph{Hyperparameters.} Unless otherwise noted, we use the following ALM schedule across all experiments: initial penalty $\lambda_0=1.0$, growth rate $\rho=1.5$, Adam optimizer with learning rate $5\times 10^{-2}$. The number of outer iterations is $K_{\mathrm{outer}}=20$ for Tier~1 and ablations, $K_{\mathrm{outer}}=10$ for Tiers~2--3. Inner steps per outer iteration match the base generator's training budget (typically 100--500 steps depending on the generator and dataset size). Monotonicity perturbation magnitude is $\delta=0.5$ throughout.

\section{Theory Connection Diagram}
\label{app:theory_connection}

\begin{figure}[ht]
\centering
\begin{tikzpicture}[
  box/.style={draw, rounded corners, minimum height=0.8cm,
              minimum width=2.0cm, align=center, font=\scriptsize, thick},
  arr/.style={-Stealth, very thick, accentblue},
  gaparr/.style={-Stealth, very thick, riskred, dashed}
]
  \node[box, fill=lightred] (ci)
    {$\Omega_{\Kcal}(Q_\phi) \to 0$\\[1pt]
     $\mathcal{U}$ minimized\\[-1pt]
     {\tiny (Thm~\ref{thm:penalty})}};
  \node[box, fill=lightyellow, right=0.9cm of ci] (cond)
    {$P_j(\cdot\mid\mathrm{Pa}_j)$\\[1pt]
     $\approx Q_j(\cdot\mid\mathrm{Pa}_j)$};
  \node[box, fill=lightgreen, right=0.9cm of cond] (joint)
    {$\TV(P,Q)$\\[1pt]
     $\le \sum_j \varepsilon_j$\\[-1pt]
     {\tiny (Prop~\ref{prop:chain_tv})}};

  \draw[gaparr] (ci) --
    node[above, font=\tiny, riskred]{heuristic} (cond);
  \draw[arr] (cond) --
    node[above, font=\tiny]{formal} (joint);
\end{tikzpicture}
\caption{Connecting the theoretical guarantees.
Theorem~\ref{thm:penalty} ensures penalties vanish;
Proposition~\ref{prop:chain_tv} converts conditional closeness to joint TV control.
The intermediate step (dashed red) is a design motivation: residualized HSIC penalties are a proxy for conditional matching, not a formal equivalence.}
\label{fig:theory_connection}
\end{figure}
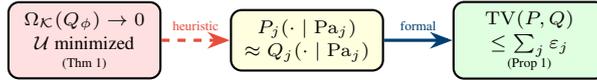

\section{Additional Experimental Results}
\label{app:experiments}

\subsection{Tier~1: Full LG Conventional Metrics}
\label{app:tier1_full}

\begin{table*}[ht]
\centering
\caption{\textbf{Tier 1: CausalWrap on the LG SCM} ($d=10$, $n=5{,}000$, 5 seeds). Each pair shows base $\to$ +CW. $\Delta$\,\% row: \impup{green} = improvement, \impdn{red} = degradation. Underline marks the largest ATE error reduction among wrapped bases.}
\label{tab:lg_results}
\footnotesize
\begin{tabular}{l ccc cc}
\toprule
& \multicolumn{3}{c}{Conventional} & \multicolumn{2}{c}{Causal} \\
\cmidrule(lr){2-4} \cmidrule(lr){5-6}
Generator & MMD$\downarrow$ & JSD$\downarrow$ & TSTR$\uparrow$ & ATE err$\downarrow$ & CI Pass$\uparrow$ \\
\midrule
CTGAN         & 0.023 & 0.052 & 0.744 & 0.832 & 0.669 \\
$\Delta$ CW   & \impup{$-$72.7\%} & \impup{$-$31.7\%} & \impup{$+$1.0\%} & \impup{\underline{$-9.1\%$}} & \impup{$+$4.7\%} \\[3pt]
TabDDPM       & 0.011 & 0.020 & 0.763 & 0.644 & 0.486 \\
$\Delta$ CW   & \impup{$-$39.2\%} & \impup{$-$7.6\%} & \impup{$+$0.3\%} & \impup{$-$1.7\%} & \impup{$+$34.1\%} \\[3pt]
TVAE          & 0.011 & 0.056 & 0.761 & 0.766 & 0.916 \\
$\Delta$ CW   & \impup{$-$24.5\%} & \impup{$-$36.4\%} & \impdn{$-$0.8\%} & \impup{$-$6.3\%} & \impdn{$-$13.3\%} \\
\midrule
Oracle \scriptsize{(true SCM)} & \textbf{0.003} & \textbf{0.007} & \textbf{0.769} & \textbf{0.572} & \textbf{0.947} \\
\bottomrule
\end{tabular}
\end{table*}

\subsection{Forbidden-Edge Ablation}
\label{app:e0_ablation}

Table~\ref{tab:e0_ablation} compares CausalWrap on the ICU cohort under two knowledge configurations: $E^0=\emptyset$ (only trusted edges and monotonicity) versus the full $\Kcal$ with temporal forbidden edges. Adding $E^0$ improves TabDDPM's ATE agreement gain ($+0.38$ vs.\ $+0.24$) but weakens the gains for CTGAN (which flips from $+0.19$ to $-0.09$) and TVAE ($+0.31\to+0.13$). The likely mechanism is \emph{constraint budget dilution}: enforcing edge removal on a frozen model requires more correction to the base residual than correcting ATE bias, and HSIC penalties for forbidden edges that the base generator already approximately satisfies consume optimization capacity without providing useful gradient signal, leaving less room for the beneficial monotonicity and moment-matching terms.

\begin{table}[ht]
\centering
\caption{\textbf{Forbidden-edge ablation} on the ICU cohort (5 seeds). $\Delta$CW reports the change from base to +CW within each knowledge configuration. $E^0{=}\emptyset$: trusted edges and monotonicity only. Full $\Kcal$: adds 7 temporal forbidden edges.}
\label{tab:e0_ablation}
\footnotesize
\setlength{\tabcolsep}{2.5pt}
\begin{tabular}{@{}l ccc ccc@{}}
\toprule
& \multicolumn{3}{c}{$E^0 = \emptyset$} & \multicolumn{3}{c}{Full $\Kcal$ (with $E^0$)} \\
\cmidrule(lr){2-4} \cmidrule(lr){5-7}
Base & $\Delta$MMD & $\Delta$TSTR & $\Delta$ATE & $\Delta$MMD & $\Delta$TSTR & $\Delta$ATE \\
\midrule
CTGAN  & \impup{$-$69\%} & \impup{$+$0.1\%} & \impup{$+$.19} & \impup{$-$64\%} & \impdn{$-$1.5\%} & \impdn{$-$.09} \\
TabDDPM & \impdn{$+$72\%} & \impdn{$-$17\%}  & \impup{$+$.24} & \impdn{$+$61\%} & \impdn{$-$28\%}  & \impup{$+$.38} \\
TVAE   & \impup{$-$6\%}  & \impdn{$-$0.4\%} & \impup{$+$.31} & \impdn{$+$40\%} & \impdn{$-$0.6\%} & \impup{$+$.13} \\
\bottomrule
\end{tabular}
\end{table}

\subsection{Constraint-Type and Knowledge-Fraction Ablations}
\label{app:constraint_knowledge_ablations}

\begin{figure*}[ht]
    \centering
    \begin{subfigure}[b]{0.44\linewidth}
        \includegraphics[width=\linewidth]{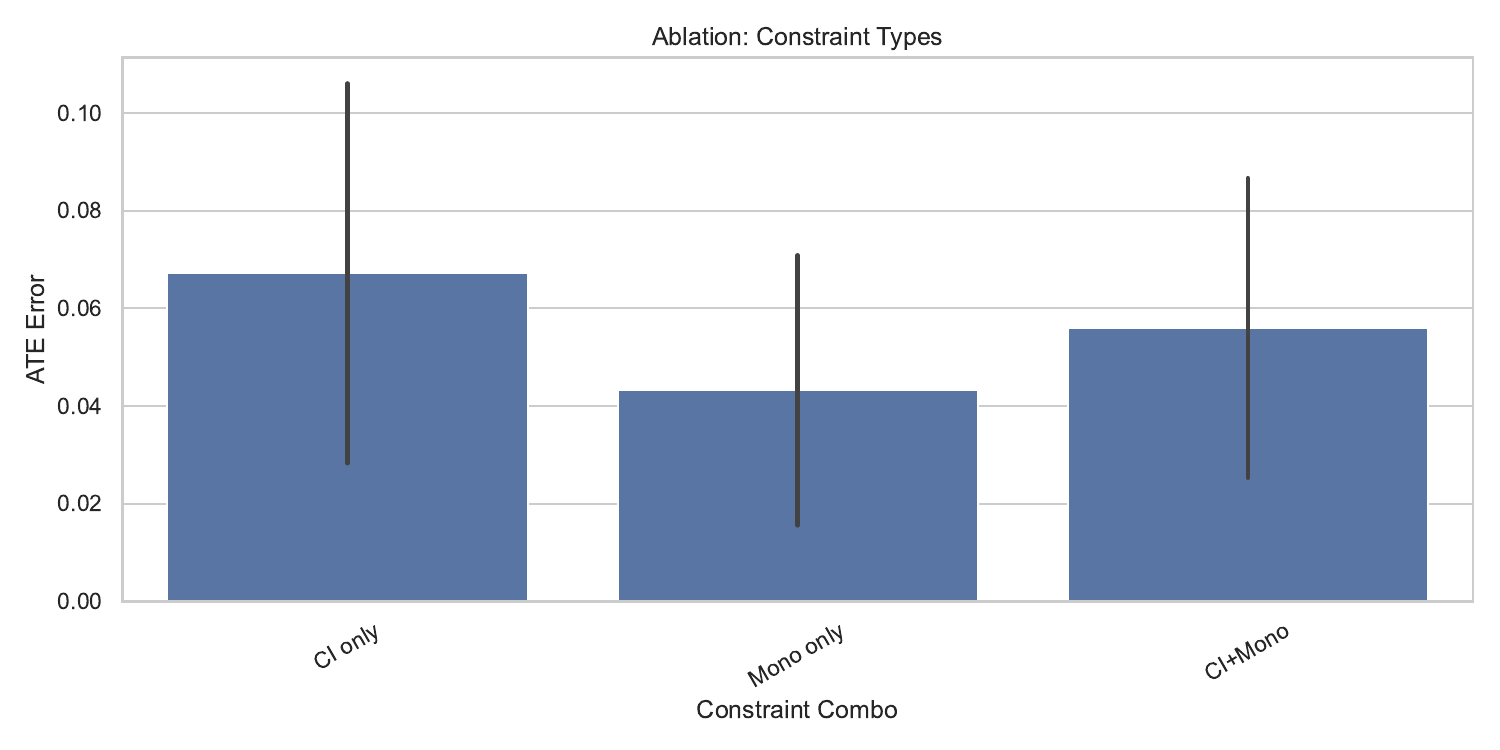}
        \caption{Constraint type ablation}
        \label{fig:ablation_constraints}
    \end{subfigure}
    \hfill
    \begin{subfigure}[b]{0.44\linewidth}
        \includegraphics[width=\linewidth]{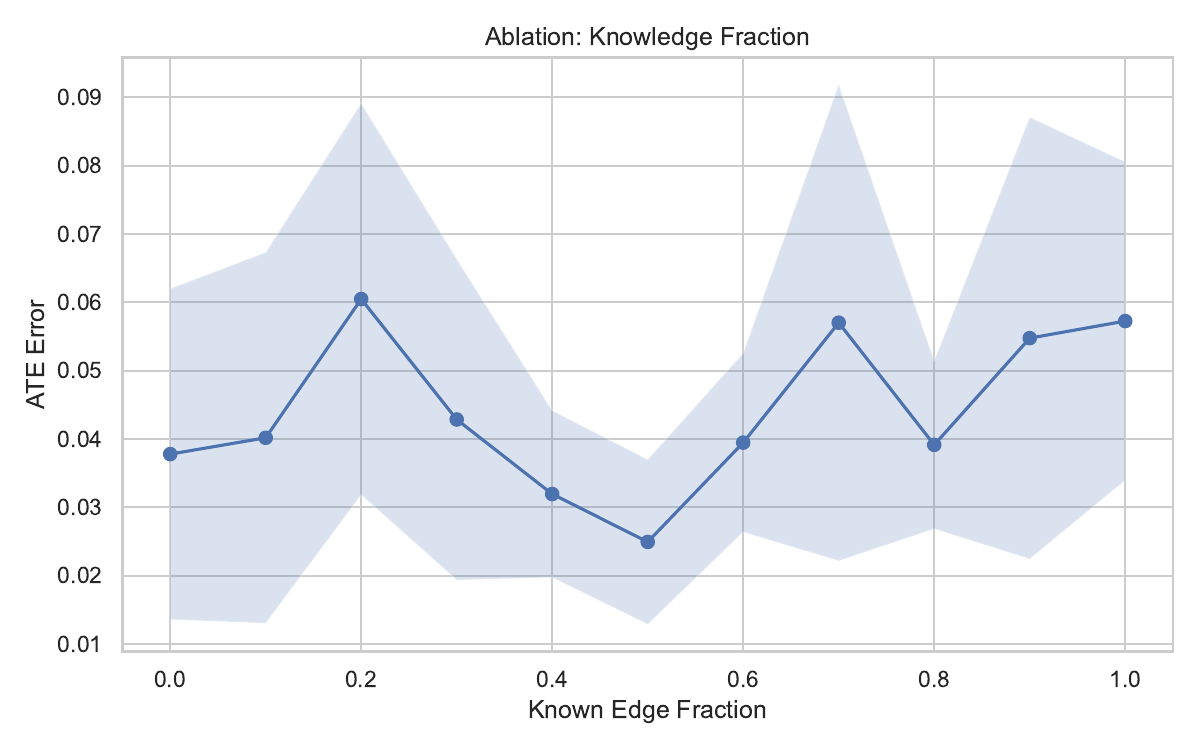}
        \caption{Knowledge fraction}
        \label{fig:ablation_knowledge}
    \end{subfigure}
    \caption{\textbf{Ablation studies} on the Linear-Gaussian SCM with TabDDPM base.}
\end{figure*}

\subsection{Additional Ablation Studies}
\label{app:ablations}

\begin{figure*}[ht]
    \centering
    \begin{subfigure}[b]{0.48\linewidth}
        \includegraphics[width=\linewidth]{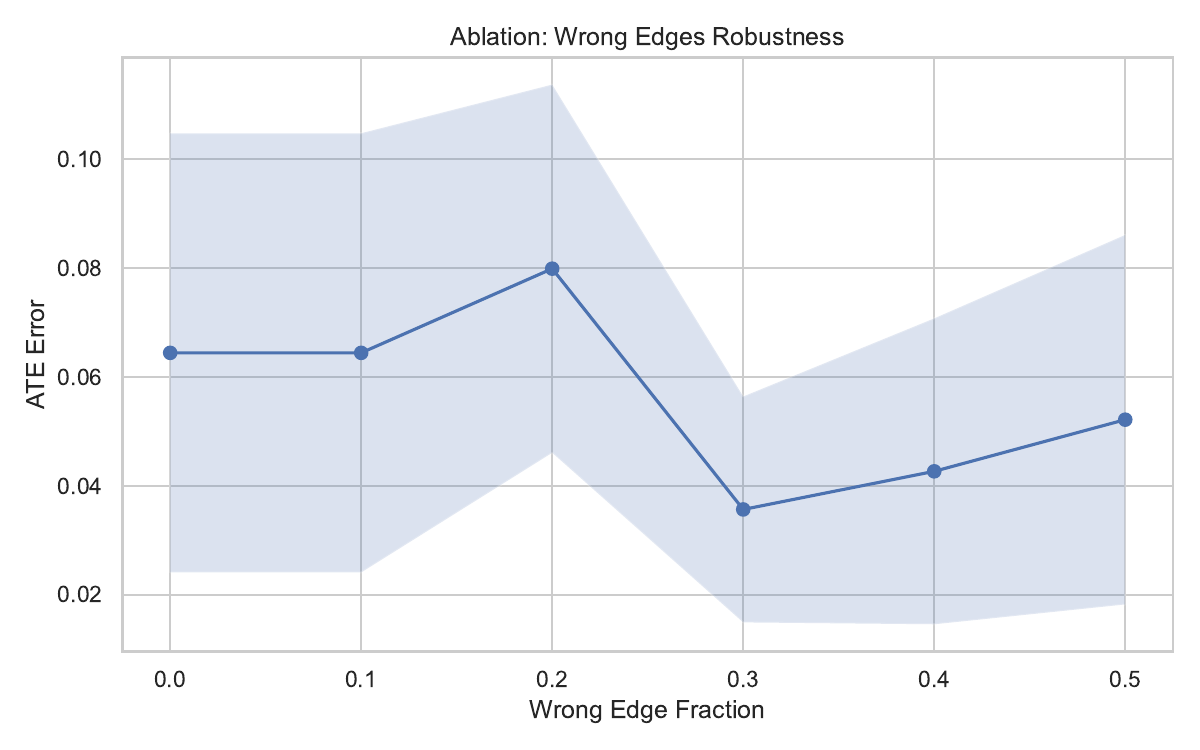}
        \caption{Robustness to wrong edges}
        \label{fig:ablation_wrong}
    \end{subfigure}
    \hfill
    \begin{subfigure}[b]{0.48\linewidth}
        \includegraphics[width=\linewidth]{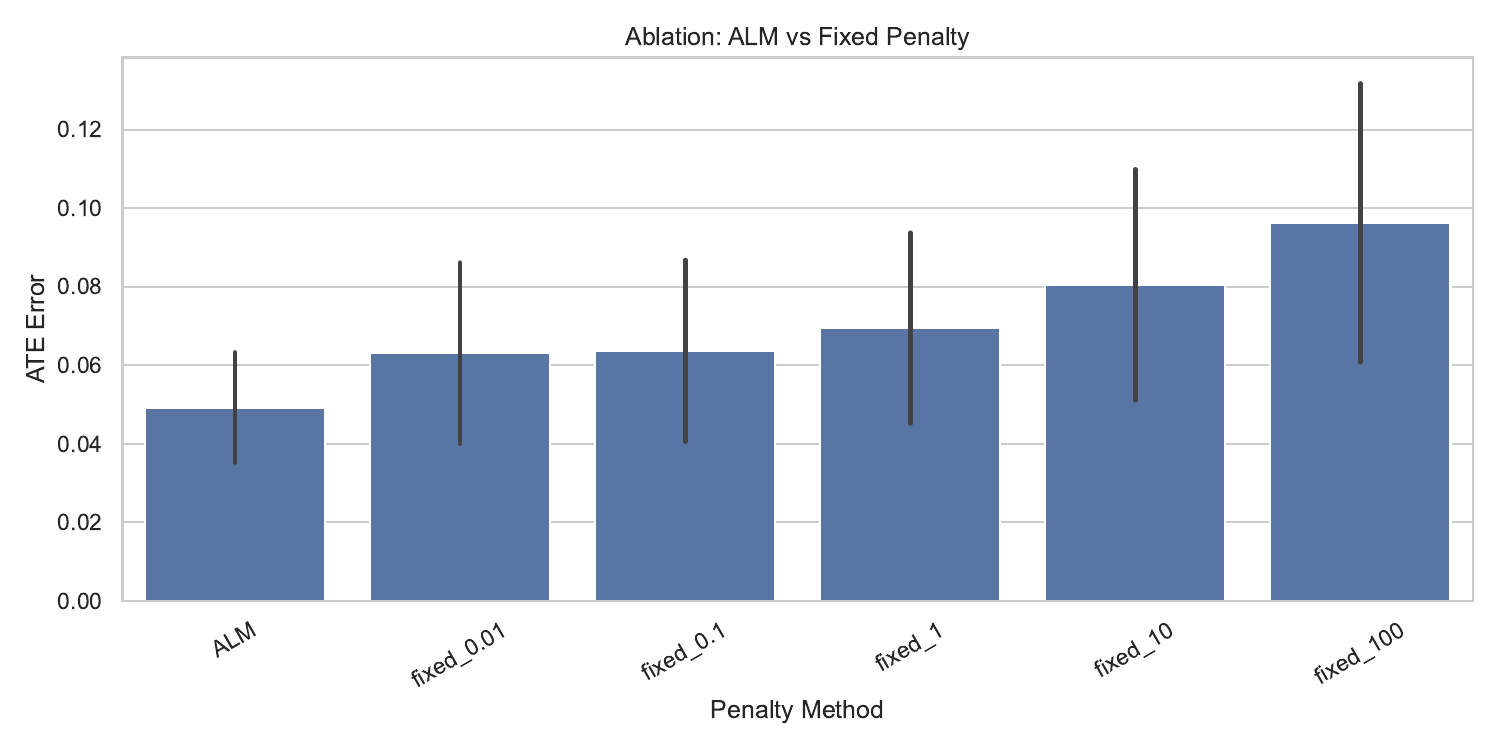}
        \caption{ALM vs.\ fixed $\lambda$}
        \label{fig:ablation_alm}
    \end{subfigure}
    \caption{\textbf{Additional ablation studies} on the Linear-Gaussian SCM with TabDDPM base.}
\end{figure*}

\end{document}